\documentclass[runningheads]{llncs}
\usepackage[T1]{fontenc}
\usepackage{amssymb,amsmath}
\usepackage{mathtools}
\usepackage[scr=rsfs]{mathalpha}
\usepackage{booktabs,tabularray}
\usepackage{enumitem}
\usepackage[dvipsnames]{xcolor}
\usepackage{hyperref}
\usepackage{graphicx,subcaption}
\usepackage{algorithm,algorithmicx,algpseudocode}
\algnewcommand{\IIf}[1]{\State\algorithmicif\ #1\ \algorithmicthen}
\algnewcommand{\EndIIf}{\unskip\ \algorithmicend\ \algorithmicif}
\usepackage[open,openlevel=2,atend,numbered]{bookmark}
\usepackage[capitalise,noabbrev]{cleveref}
\usepackage{tikz}
\usepackage{orcidlink}
\crefname{property}{Property}{Properties}
\spnewtheorem{assumption}[definition]{Assumption}{\bfseries}{\itshape}
\crefname{assumption}{Assumption}{Assumptions}
\newcommand{\regret}{\textup{Reg}}
\newcommand{\dualgap}{\textup{D-Gap}}
\newcommand{\neregret}{\textup{NE-Reg}}
\newcommand{\term}[1]{(\theequation\text{#1})}
\newcommand{\termref}[2]{\text{Equation~(\hyperref[#1]{\ref{#1}#2})}}
\newcommand{\symit}{\mathit}
\newcommand{\citet}{\cite}
\newcommand{\citep}{\cite}
\begin{document}
\title{Proximal Point Method for Online Saddle Point Problem}
%
%
\author{Qing-xin Meng\,\orcidlink{0000-0003-4014-7405} \and
Jian-wei Liu\thanks{Corresponding author}}
\authorrunning{Q. Meng and J. Liu}
%
\institute{Department of Automation, College of Artificial Intelligence\\
China University of Petroleum, Beijing, China\\
\email{qingxin6174@gmail.com, liujw@cup.edu.cn}}
\maketitle              
\begin{abstract}
This paper focuses on the online saddle point problem, which involves a sequence of two-player time-varying convex-concave games. Considering the nonstationarity of the environment, we adopt the duality gap and the dynamic Nash equilibrium regret as performance metrics for algorithm design. We present three variants of the proximal point method: the Online Proximal Point Method (OPPM), the Optimistic OPPM~(OptOPPM), and the OptOPPM with multiple predictors. Each algorithm guarantees upper bounds for both the duality gap and dynamic Nash equilibrium regret, achieving near-optimality when measured against the duality gap. Specifically, in certain benign environments, such as sequences of stationary payoff functions, these algorithms maintain a nearly constant metric bound. Experimental results further validate the effectiveness of these algorithms. Lastly, this paper discusses potential reliability concerns associated with using dynamic Nash equilibrium regret as a performance metric. 
The technical appendix and code can be found at \cite{2024arXiv240704591M} and \url{https://github.com/qingxin6174/PPM-for-OSP}.

\keywords{Online Saddle Point Problem \and
Proximal Point Method \and
Duality Gap \and
Multiple Predictors.}
\end{abstract}

\section{Introduction}

The Online Saddle Point~(OSP) problem, initially introduced by \citet{cardoso2018online}, involves a sequence of two-player time-varying convex-concave games. 
In round $t$, Players~1 and~2 \emph{jointly} select a strategy pair $(x_t,y_t)\in X\times Y$. Here, Player~1 minimizes his payoff, while Player~2 maximizes his payoff. 
Both players make decisions without prior knowledge of the current or future payoff functions. 
Upon finalizing the strategy pair, the environment reveals a continuous payoff function $f_t\colon X\times Y\rightarrow\mathbb{R}$, which satisfies the following conditions: $\forall y\in Y$, $f_t\left(\,\cdot\,,y\right)$ is convex on $X$, and $\forall x \in X$, $f_t\left(x, \,\cdot\,\right)$ is concave on $Y$. 
No additional assumptions are imposed on the environment, thereby allowing potential regularity or even adversarial behavior.
The goal is to provide players with decision-making algorithms that approximate Nash equilibria, ensuring that the players' decisions in most rounds are close to saddle points. 

Observe that the OSP problem extends the application of Online Convex Optimization~(OCO, \cite{zinkevich2003online}) to include interactions among two players and the environment.  Consequently, it becomes straightforward to identify online scenarios that are well-suited for OSP. Such scenarios include dynamic routing~\cite{Awerbuch2008online,Guo2021Routing} and online advertising auctions~\cite{Lykouris2016learning}, which fall under the broader category of budget-reward trade-offs. 

Given the nonstationarity of the environment, there are two optional metrics available for evaluating the performance of an OSP algorithm: 
\begin{itemize}
\item[1)] The duality gap~(D-Gap, \cite{zhang2022noregret}), which is given by 
\begin{equation}
\label{def:dual-gap}
\begin{aligned}
\dualgap_T\coloneq\sum_{t=1}^T\max_{y\in Y} f_t\left(x_t, y\right)-\sum_{t=1}^T\min_{x\in X}f_t\left(x, y_t\right).
\end{aligned}
\end{equation}
\item[2)] The dynamic Nash equilibrium regret~(NE-Reg, \cite{zhang2022noregret}), which is defined as 
\begin{equation}
\label{def:NE-regret}
\begin{aligned}
\neregret_{T}
\coloneq\left\lvert \sum_{t=1}^T f_t\left(x_t, y_t\right)-\sum_{t=1}^T \max_{y\in Y}\min_{x\in X}f_t\left(x, y\right)\right\rvert.
\end{aligned}
\end{equation}
\end{itemize}

\citet{zhang2022noregret} first presented an algorithm that, under bilinear payoff functions, simultaneously guarantees upper bounds for three performance metrics: the duality gap, dynamic Nash equilibrium regret, and static individual regret. 
In their commendable work, the authors have directed their focus towards the algorithm's adaptability to a spectrum of metrics, attenuating the pursuit of algorithmic optimality. 
Notably, their method yields a duality gap upper bound of $\widetilde{O}\big(\sqrt{T}\big)$ for sequences of stationary payoff functions, whereas we advocate for a sharper bound of $O(1)$ that we believe is optimal in such scenarios. 
Our assertion is inspired by \citet{Campolongo2020Temporal}, who demonstrated that proximal methods can incur $O(1)$-level regret in online convex optimization when dealing with stationary loss function sequences.

\begin{table}[t]
\centering\setlength{\tabcolsep}{1em}
\caption{Summary of our results. In this table, $\widetilde{O}$ hides poly-logarithmic factors, $V_T$ represents the cumulative difference of the convex-concave payoff function series, $V'_T$ denotes the cumulative error of one single predictor, and $V_T^k$ indicates the cumulative error of the $k$-th predictor. }
\label{tab:results}
\begin{tabular}{ll}
\toprule
Algorithm & Upper Bound of $\dualgap_T$ and $\neregret_T$ \\
\midrule
OPPM & $\widetilde{O}\left(\min\left\{ V_T, \sqrt{(1+C_T)T}\right\}\right)$ \\[3pt]
OptOPPM & $\widetilde{O}\left(\min\left\{ V'_T, \sqrt{(1+C_T)T}\right\}\right)$ \\[3pt]
OptOPPM with multiple predictors & $\widetilde{O}\left(\min\left\{ V^1_T, \cdots, V^d_T, \sqrt{(1+C_T)T}\right\}\right)$ \\
\bottomrule
\end{tabular}
\end{table}

\subsubsection{Contributions} In this paper, we propose three variants of the proximal point method for solving the OSP problem: the Online Proximal Point Method~(OPPM), the Optimistic OPPM~(OptOPPM), which allows for an arbitrary predictor, and OptOPPM with multiple predictors, enhancing the algorithm to handle multiple predictors. 
Results are shown in \cref{tab:results}. 
All algorithms are near-optimal, as they achieve a worst-case duality gap upper bound of $\widetilde{O}\big(\sqrt{(1+C_T)T}\big)$. 
In particular, under favorable scenarios such as stationary payoff function sequences, these algorithms attain a sharp bound of $\widetilde{O}(1)$. 
Notably, the OptOPPM with multiple predictors can autonomously select the most effective predictor from a set of $d$ options. 
Even when all predictors underperform, it preserves the worst-case bound, significantly enhancing the algorithm's practical utility. 

It is imperative to highlight that in time-varying games, relying on NE-Reg for performance evaluation of algorithms may lead to concerns over its reliability. 
A sublinear D-Gap suggests an approximation to a coarse correlated equilibrium, while a sublinear NE-Reg does not necessarily imply such a correlation. 
Consider a scenario where, in half of the iterations, the actual payoffs significantly exceed those at the Nash equilibrium, while in the other half, they are substantially lower. 
This results in a small NE-Reg, but it does not guarantee that the strategies in most rounds are close to the Nash equilibrium. 
See \cref{exp:ne-reg} for a more precise discussion. 

\subsubsection{Related Work}
The OSP problem is a time-varying version of the minimax problem. The first minimax theorem was proven by \citet{vNeumann1928Zur}. 
Subsequent to the seminal work of \citet{FREUND199979}, which unveiled the connections between the minimax problem and online learning, there has been a burgeoning interest in the development of no-regret algorithms tailored for static environments~\cite{Anagnostides2022near,DASKALAKIS2015327,daskalakis2021nearoptimal,nguyen2019exploiting,Rakhlin2013Optimization,Syrgkanis2015fast}. 
In recent years, the research focus has expanded to encompass the OSP problem and its various variants~\cite{anagnostides2023on,cardoso2018online,cardoso2019competing,fiez2021online,Roy2019online,zhang2022noregret}. 

\citet{cardoso2018online} were the pioneers in explicitly addressing the OSP problem, introducing the saddle-point regret as $\big\lvert \sum_{t=1}^T f_t\left(x_t, y_t\right)-\min_{x\in X}\max_{y\in Y}\sum_{t=1}^T f_t\left(x,y\right)\big\rvert$. 
In a subsequent work, \citet{cardoso2019competing} redefined saddle-point regret as Nash equilibrium regret and developed an FTRL-like algorithm capable of achieving a Nash equilibrium regret bound of $\widetilde{O}(\sqrt{T})$. 
Moreover, they proved that it is impossible for any algorithm to simultaneously attain sublinear Nash equilibrium regret and sublinear individual regret for both players. 
Building on these findings, \citet{zhang2022noregret} refined the notion of dynamic Nash equilibrium regret by reassessing the Nash equilibrium regret, moving the minimax operation inside the summation, and delineating it as \cref{def:NE-regret}. 
They also proposed an algorithm predicated on the meta-expert framework, which ensures upper bounds for three performance metrics: the duality gap, dynamic Nash equilibrium regret, and static individual regret, effectively covering beliefs from stationary to highly nonstationary. 

In contrast to the approach by \citet{zhang2022noregret}, which targets a broad spectrum of nonstationarity levels, this paper take beliefs that the environment exhibits nonstationarity and accordingly designs algorithms to achieve near-optimal performance. 
Moreover, this study underscores the potential issues with the reliability of dynamic Nash equilibrium regret as a metric for evaluating algorithmic performance in time-varying games.

\section{Preliminaries}

Throughout this paper, we define $(\,\cdot\,)_+\coloneq\max(\,\cdot\,,0)$, and use the abbreviated notation $1\!:\!n$ to represent $1,2,\cdots,n$. 
For asymptotic upper bounds, we employ big $O$ notation, while $\widetilde{O}$ is utilized to hide poly-logarithmic factors. 
Denote by $\langle\,\cdot\,,\cdot\,\rangle\colon \mathscr{X}^*\times \mathscr{X}\rightarrow\mathbb{R}$ the canonical dual pairing, where $\mathscr{X}^*$ represents the dual space of $\mathscr{X}$. 

The Fenchel coupling~\cite{Mertikopoulos2016Learning,2016arXiv160807310M} induced by a proper function $\varphi$ is defined as 
\begin{equation*}
B_{\varphi}\left(x, z\right)\coloneq\varphi\left(x\right)+\varphi^\star\left(z\right)-\left\langle z, x\right\rangle, \quad\forall\left(x, z\right)\in \mathscr{X}\times \mathscr{X}^*, 
\end{equation*}
where $\varphi^\star$ represents the convex conjugate of $\varphi$ given by $\varphi^\star\left(z\right)\coloneq\sup_{x\in \mathscr{X}} \langle z, x\rangle-\varphi\left(x\right)$. 
Fenchel coupling is the general form of Bregman divergence. 
By Fenchel-Young inequality, $B_{\varphi}\left(x, z\right)\geq 0$, and equality holds iff $z$ is the subgradient of $\varphi$ at $x$. 
We use the superscripted notation $x^\varphi$ to abbreviate one subgradient of $\varphi$ at $x$. 
Directly applying the definition of Fenchel coupling yields $B_{\varphi}\left(x, y^\varphi\right)+B_{\varphi}\left(y, z\right)-B_{\varphi}\left(x, z\right)=\left\langle z-y^\varphi, x-y\right\rangle$. 
A similar version can be found in \cite{chen1993convergence}. 

$\varphi$ is $\mu$-strongly convex if 
\begin{equation*}
B_{\varphi}\left(x, y^\varphi\right)\geq\frac{\mu}{2}\left\lVert x-y\right\rVert^2,\qquad\forall x\in \mathscr{X},\ \forall y^\varphi\in\partial\varphi\left(y\right). 
\end{equation*}


\section{Main Results}

This section begins by outlining the assumptions related to the OSP problem. 
It then evaluates performance metrics and establishes the lower bound for the duality gap, a crucial step in analyzing algorithmic optimality. 
Finally, this section introduces three variants of the proximal point method: the Online Proximal Point Method~(OPPM), the Optimistic OPPM~(OptOPPM) that incorporates an arbitrary predictor, and a variant of OptOPPM designed to accommodate multiple predictors, thereby enhancing the algorithm's practical utility. 
Refer to \cite{2024arXiv240704591M} 
for all theorem proofs. 

\subsection{Assumptions}

Let $\big(\mathscr{X}, \lVert\cdot\rVert_\mathscr{X}\big)$ and $\big(\mathscr{Y}, \lVert\cdot\rVert_\mathscr{Y}\big)$ be normed vector spaces. 
We can omit norm subscripts without ambiguity when the norm is determined by the space to which the element belongs. 
Let $X\subset\mathscr{X}$ and $Y\subset\mathscr{Y}$. 
Now we introduce the following two assumptions: 
%
\begin{assumption}
\label{ass:X-Y-bounded}
The feasible sets $X$ and~\,$Y$ are both compact and convex, with the diameter of~\,$X$ denoted as $D_X$, and the diameter of~\,$Y$ as $D_Y$.
\end{assumption}
%
\begin{assumption}
\label{ass:2-subgradient-bounded}
All payoff functions satisfy the boundedness of subdifferentiation, i.e., $\exists G_X$, $G_Y<+\infty$, $\forall x\in X$, $\forall y\in Y$, $\forall t$, $\left\lVert \partial_x f_t\left(x,y\right)\right\rVert\leq G_X$, $\lVert \partial_y (-f_t)\left(x,y\right)\rVert\leq G_Y$. 
\end{assumption}
%
In accordance with Theorem~3 by \citet{kakutani1941generalization}, \cref{ass:X-Y-bounded} ensures the existence of a saddle point. Specifically, there exists a pair $\left(x_t^*, y_t^*\right)\in X\times Y$ such that for all $x\in X$ and $y\in Y$: 
$f_t(x_t^*, y)\leq f_t(x_t^*, y_t^*)\leq f_t(x, y_t^*)$. 
One may write 
$(x_t^*, y_t^*)=\arg\min_{x\in X}\max_{y\in Y}f_t(x, y)$.

%
\subsection{Discussion on Performance Metrics}

This subsection first confirms that the individual regret~(Reg) can guarantee both the D-Gap and NE-Reg, and then proceeds to demonstrate the intrinsic issues associated with NE-Reg. 
Prior to this, let $x'_t=\arg\min_{x\in X}f_t\left(x, y_t\right)$ and $y'_t=\arg\max_{y\in Y} f_t\left(x_t, y\right)$. 
Individual regrets of Players~1 and~2 are defined as:
\begin{equation*}
\begin{aligned}
\regret^1_T\coloneq\sum_{t=1}^T f_t\left(x_t, y_t\right)-\sum_{t=1}^T f_t\left(x'_t, y_t\right), \quad
\regret^2_T\coloneq\sum_{t=1}^T f_t\left(x_t, y'_t\right)-\sum_{t=1}^T f_t\left(x_t, y_t\right).
\end{aligned}
\end{equation*}

$\dualgap_T=\regret^1_T+\regret^2_T$ implies the following lemma. 
\begin{lemma}
\label{lem:reg-dualgap}
If an online algorithm guarantees individual regrets for both players, it also ensures a $\dualgap$ guarantee. 
\end{lemma}
Considering that two individual regrets provide the same evaluative utility as the D-Gap, we choose to use only the D-Gap as the metric to assess an online algorithm's performance. 

$\neregret_T\leq\dualgap_T$ (Proposition~11 of \cite{zhang2022noregret}) deduces the following lemma. 
\begin{lemma}
\label{lem:reg-nereg}
If an online algorithm offers individual regret guarantees for both players, it also ensures a guarantee of $\neregret$.
\end{lemma}
It is noteworthy that using NE-Reg as a performance metric could raise questions about its reliability. 
The following example illustrates that the environment may intentionally mislead players, resulting in an internal cancellation of the absolute value of NE-Reg.
\begin{example}
\label{exp:ne-reg}
Consider the OSP problem defined on the feasible domain $[-1,1]^2$. 
In round $t$, the environment feeds back the payoff function of $f_t(x,y)=(x-x_t^*)^2-(y-y_t^*)^2$, where 
\begin{equation*}
\begin{aligned}
x_t^*=x_t+\frac{(-1)^t+1}{2}(2\chi_{x_t<0}-1), \qquad
y_t^*=y_t-\frac{(-1)^t-1}{2}(2\chi_{y_t<0}-1),
\end{aligned}
\end{equation*}
$\chi_A$ is the 0\,/\,1 indicator function with $\chi_A=1$ if and only if $A$ is true. 
Regardless of the algorithm used, the players' strategy pair will never approach the saddle point, but $\neregret_T=\left|\sum_{t=1}^T(-1)^t\right|\leq 1$. 
\end{example}
Although Remark~2 in \citet{zhang2022noregret} pointed out that performance metrics based on function values inherently possess certain limitations, the weakness of NE-Reg is easier to trigger.

\subsection{Duality Gap Lower Bound}

Given the reliability concerns of $\neregret$, this subsection only demonstrates the lower bound for the $\dualgap$. 

\begin{theorem}[Duality Gap Lower Bound]
\label{thm:lower-bound}
Regardless of the strategy pairs adopted by the two players, there always exist a sequence of convex-concave payoff functions $f_{1:T}$ that adhere to \cref{ass:X-Y-bounded,ass:2-subgradient-bounded}, ensuring a duality gap of at least $\symit{\Omega}\big(\sqrt{(1+C_T)T}\big)$, where $C_T=\sum_{t=1}^T\bigl(\left\lVert x'_t-x'_{t-1}\right\rVert+\left\lVert y'_t-y'_{t-1}\right\rVert\bigr)$. 
\end{theorem}

An online algorithm that attains a D-Gap upper bound of $\widetilde{O}\big(\sqrt{(1+C_T)T}\big)$ is considered near-optimal, signifying that the upper bound matches the lower bound up to poly-logarithmic terms. 
However, achieving this upper bound is trivial.
Specifically, if Players~1 and~2 independently select OCO algorithms, such as Ader~\cite{zhang2018adaptive} or Online Mirror Descent with doubling trick, they can ensure that $\regret^1_T\leq \widetilde{O}\big(\sqrt{\big(1+\sum_{t=1}^T\lVert x'_t-x'_{t-1}\rVert\big)T}\big)$, $\regret^2_T\leq \widetilde{O}\big(\sqrt{\big(1+\sum_{t=1}^T\lVert y'_t-y'_{t-1}\rVert\big)T}\big)$, 
thereby securing $\dualgap_T\leq \widetilde{O}\big(\sqrt{(1+C_T)T}\big)$. 

In the following three subsections, we will design non-trivial algorithms that not only ensure an $\widetilde{O}\big(\sqrt{(1+C_T)T}\big)$ $\dualgap$ upper bound but also further reduce the $\dualgap$ in certain benign environments.

\subsection{Online Proximal Point Method}

The proximal point method, initially introduced in the seminal work by \citet{Rockafellar1976Monotone}, has established itself as a classic first-order method for solving minimax problems. 
To tailor it for the solution of the OSP problem, this subsection introduces and analyzes the \emph{Online Proximal Point Method}~(OPPM). 
OPPM can be formalized as 
\begin{equation*}
\begin{aligned}
(x_{t+1},y_{t+1})=\arg\min_{x\in X}\max_{y\in Y} f_t(x,y)+\frac{1}{\eta_t}B_{\phi}\big(x, x_t^\phi\big)-\frac{1}{\gamma_t}B_{\psi}\big(y, y_t^\psi\big), 
\end{aligned}
\end{equation*}
where $\eta_t>0$ and $\gamma_t>0$ represent learning rates of Players~$1$ and~$2$, respectively, $x_{t}^\phi\in\partial\phi\left(x_{t}\right)$, $y_{t}^\psi\in\partial\psi\left(y_{t}\right)$. 
To facilitate our analysis, we assume that the regularization terms satisfy the following property. 
\begin{property}
\label{pro:2-Bregman-Lipschitz}
The functions $\phi$ and $\psi$ are both $1$-strongly convex, and their Fenchel couplings satisfy Lipschitz continuity for the first variable. 
Specifically, $\exists L_\phi<+\infty$, $\exists L_\psi<+\infty$, $\forall\alpha, x, x'\in X$, $\forall\beta, y, y'\in Y$: 
\begin{equation*}
\begin{aligned}
\left\lvert B_\phi(x,\alpha^\phi)-B_\phi(x',\alpha^\phi)\right\rvert\leq L_\phi\left\lVert x-x'\right\rVert,\ 
\left\lvert B_\psi(y,\beta^\psi)-B_\psi(y',\beta^\psi)\right\rvert\leq L_\psi\left\lVert y-y'\right\rVert.
\end{aligned}
\end{equation*}
\end{property}

The following theorem provides the individual regret guarantee for OPPM.
\begin{theorem}[Individual Regret for OPPM]
\label{lem:OPPM}
Under \cref{ass:X-Y-bounded,ass:2-subgradient-bounded}, let regularizers satisfy \cref{pro:2-Bregman-Lipschitz}, and let $C$ be a preset upper bound of $C_T$. 
If the learning rates of two players follow from $\eta_t=L\left(2D+C\right)/\big(\epsilon+\sum_{\tau=1}^{t-2}\Delta_{\tau}\big)=\gamma_t$, 
where the constant $\epsilon > 0$ prevents initial learning rates from being infinite, $\Delta_t=(\symit{\Sigma}_t-\max_{\tau\in 1:t-1}\symit{\Sigma}_{\tau})_+$, 
$\symit{\Sigma}_{t}=\max\left\{\symit{\Sigma}_{t}^1, \symit{\Sigma}_{t}^2\right\}$, and
\begin{equation*}
\begin{aligned}
\symit{\Sigma}_T^1&=\textstyle\left(\sum_{t=1}^{T}\bigl(f_t(x_t,y_t)-f_t(x_{t+1}, y_{t+1})+f_t(x'_{t+1}, y_{t+1})-f_t(x'_{t}, y_t)\bigr)\right)_{\!+}, \\[-1pt]
\symit{\Sigma}_T^2&=\textstyle\left(\sum_{t=1}^{T}\bigl(f_t(x_t,y'_t)-f_t(x_{t+1},y'_{t+1})+f_t(x_{t+1},y_{t+1})-f_t(x_t,y_t)\bigr)\right)_{\!+}.
\end{aligned}
\end{equation*}
Then OPPM achieves $\regret^1_T, \regret^2_T \leq O\big(\min\big\{\sum_{t=1}^{T}\rho\left(f_{t}, f_{t-1}\right), \sqrt{(1+C)T}\big\}\big)$, 
where $\rho\left(f_{t}, f_{t-1}\right)=\max_{x\in X,\,y\in Y}\left\lvert f_t(x,y)-f_{t-1}(x,y)\right\rvert$ measures the distance between $f_t$ and $f_{t-1}$. 
\end{theorem}

\cref{lem:OPPM} corresponds to Theorem 5.1 of \citet{Campolongo2021closer}, but involves more complex parameters and a complicated proof (see Appendix~B in \cite{2024arXiv240704591M}) due to the mutual influence between players from the joint update of OPPM. This joint update is crucial. Independent execution of implicit online mirror descent, as per \citet{Campolongo2021closer}, only ensures the worst-case bound without addressing the temporal variability term. 

In \cref{lem:OPPM}, learning rates depend on the preset value $C$, but this dependency can be avoided by using the doubling trick~\cite{schapire1995gambling}. The pseudocode for OPPM is in \cref{alg:OPPM}. 
Applying \cref{lem:reg-dualgap,lem:reg-nereg}, we deduce that \cref{alg:OPPM} offers assurances on both the D-Gap and NE-Reg. 



%
\begin{algorithm}[t]
\caption{OPPM with Adaptive Learning Rates}
\label{alg:OPPM}
\algrenewcommand\algorithmicensure{\textbf{Initialize:}}
\begin{algorithmic}[1]
\Require Feasible sets $X$ and $Y$ satisfy \cref{ass:X-Y-bounded}. The payoff function $f_t$ satisfies \cref{ass:2-subgradient-bounded}. Regularizers satisfy \cref{pro:2-Bregman-Lipschitz}. 
\Ensure $C$, $\epsilon > 0$. 
\For{$t=1,2,\cdots$ }
\State Output $(x_{t},y_{t})$, then observe a continuous convex-concave payoff function $f_t$
\State $x'_t=\arg\min_{x\in X}f_t\left(x, y_t\right)$, $y'_t=\arg\max_{y\in Y} f_t\left(x_t, y\right)$
\IIf{$\sum_{\tau=1}^t \big(\big\lVert x'_\tau-x'_{\tau-1}\big\rVert+\big\lVert y'_\tau-y'_{\tau-1}\big\rVert\big)> C$} $C\gets 2C$ \Comment{Doubling trick}
\State Update $\eta_t$ and $\gamma_t$ according to the setting of \cref{lem:OPPM}
\State $(x_{t+1},y_{t+1})=\arg\min_{x\in X}\max_{y\in Y} f_t(x,y)+B_{\phi}\big(x, \widetilde{x}_{t}^\phi\big)/\eta_t-B_{\psi}\big(y, \widetilde{y}_{t}^\psi\big)/\gamma_t$
\EndFor
\end{algorithmic}
\end{algorithm}

\begin{theorem}[Performance of {\cref{alg:OPPM}}]
\label{thm:OPPM}
If two players output strategy pairs according to \cref{alg:OPPM}, then
\begin{equation*}
\begin{aligned}
\dualgap_T, \neregret_T 
\leq O\left(\min\left\{V_T, \sqrt{(1+\log_2 C_T+C_T)T}\right\}\right),
\end{aligned}
\end{equation*}
where $V_T=\sum_{t=1}^{T}\rho\left(f_{t}, f_{t-1}\right)$. 
Specially, in scenarios where the sequence of payoff functions is stationary, then $\dualgap_T, \neregret_T, \regret^1_T, \regret^2_T = O(1)$. 
\end{theorem}

The worst-case bound delineated in the first statement of \cref{thm:OPPM} suggests that the OPPM approaches optimality. 
For sequences of stationary payoff functions, we observe that the D-Gap, NE-Reg, and two individual regrets all enjoy a sharp bound of $O(1)$, which surpasses the $\widetilde{O}\big(\sqrt{T}\big)$ D-Gap and slightly refines the $\widetilde{O}\left(1\right)$ NE-Reg, both reported by \citet{zhang2022noregret}, and matches the $O(1)$ static individual regret in \citet{pmlr-v134-hsieh21a}.

Typically, OPPM lacks a general solution and necessitates a case-by-case computation tailored to the specific nature of the payoff functions. 
In essence, each iteration is tasked with solving a strongly-convex-strongly-concave saddle point problem. 
In certain instances, OPPM can be expressed in a closed form, such as when the regularizers $\phi$ and $\psi$ are both square norms, and $f_t$ is bilinear or quadratic. 
When a closed form solution is unavailable, numerical techniques can be employed for efficient approximation \citep{abernethy2018faster,carmon2020coordinate,jin2022sharper}.

\subsection{Optimistic OPPM}

The performance of the OPPM algorithm is partially contingent upon the stationarity of the sequence of payoff functions. 
However, taking a highly nonstationary periodic environment as an example, proactive prediction of the upcoming payoff function can significantly improve the performance of online algorithms. 
This approach, known as `optimism', replaces the stationarity of the payoff function sequence with prediction accuracy. 
In this subsection, we explore the optimistic variant of OPPM and introduce the following update:
\begin{equation*}
\begin{aligned}
(x_{t},y_{t})&=\arg\min_{x\in X}\max_{y\in Y} h_t(x,y)+\frac{1}{\eta_t}B_{\phi}\big(x, \widetilde{x}_t^\phi\big)-\frac{1}{\gamma_t}B_{\psi}\big(y, \widetilde{y}_t^\psi\big), \\
\widetilde{x}_{t+1}&=\arg\min_{x\in X} \eta_t f_t(x,y_t)+B_{\phi}\big(x, \widetilde{x}_{t}^\phi\big), \\
\widetilde{y}_{t+1}&=\arg\max_{y\in Y} \gamma_t f_t(x_t,y)-B_{\psi}\big(y, \widetilde{y}_{t}^\psi\big),
\end{aligned}
\end{equation*}
where $h_t$ is an arbitrary convex-concave predictor. 
We refer to the above update as the \emph{Optimistic OPPM}~(OptOPPM). 

The following theorem provides the individual regret guarantee for OptOPPM.
\begin{theorem}[Individual Regret for OptOPPM]
\label{lem:OptOPPM}
Under \cref{ass:X-Y-bounded,ass:2-subgradient-bounded}, let regularizers satisfy \cref{pro:2-Bregman-Lipschitz}, let the predictor $h_t$ satisfy \cref{ass:2-subgradient-bounded}, and let $C^1$ and $C^2$ be the preset upper bounds for $C_T^1$ and $C_T^2$, respectively. 
If the learning rates of two players follow from 
\begin{equation*}
\begin{aligned}
\eta_t=\frac{L_\phi(D_X+C^1)}{\epsilon+\sum_{\tau=1}^{t-1}\delta_{\tau}^1}, 
\qquad \gamma_t=\frac{L_\psi(D_Y+C^2)}{\epsilon+\sum_{\tau=1}^{t-1}\delta_{\tau}^2},
\end{aligned}
\end{equation*}
where the constant $\epsilon > 0$ prevents initial learning rates from being infinite, and 
\begin{equation*}
\begin{aligned}
0\leq\delta_t^1=\ &f_t(x_t,y_t)-h_t(x_t,y_t)+h_t(\widetilde{x}_{t+1},y_t)-f_t(\widetilde{x}_{t+1}, y_t) -B_{\phi}\big(\widetilde{x}_{t+1}, x_t^\phi\big)/\eta_t,\\[-1pt]
0\leq\delta_t^2=\ &f_t(x_t,\widetilde{y}_{t+1})-h_t(x_t,\widetilde{y}_{t+1})+h_t(x_t,y_t)-f_t(x_t,y_t) -B_{\psi}\big(\widetilde{y}_{t+1}, y_t^\psi\big)/\gamma_t.
\end{aligned}
\end{equation*}
Then OptOPPM enjoys $\regret^i_T\!\leq\! O\big(\min\big\{\!\sum_{t=1}^{T}\rho\left(f_{t}, h_{t}\right), \sqrt{(1+C^i)T}\big\}\big)$, $i\!=\!1,2$. 
\end{theorem}
%

Similar to \cref{alg:OPPM}, the dependence on $C^1$ and $C^2$ for the learning rates in OptOPPM can be obviated by employing the doubling trick. 
Details on the adjustment of learning rates are delineated in \cref{alg:OptOPPM}.
\begin{algorithm}[t]
\caption{OptOPPM with Adaptive Learning Rates}
\label{alg:OptOPPM}
\algrenewcommand\algorithmicensure{\textbf{Initialize:}}
\begin{algorithmic}[1]
\Require Feasible sets $X$ and $Y$ satisfy \cref{ass:X-Y-bounded}. Both the payoff function $f_t$ and the predictor $h_t$ satisfy \cref{ass:2-subgradient-bounded}. Regularizers satisfy \cref{pro:2-Bregman-Lipschitz}. 
\Ensure $C^1$, $C^2$, $\epsilon > 0$.
\For{$t=1,2,\cdots$ }
\State Receive $h_t$, update $\eta_t$ and $\gamma_t$ according to the setting of \cref{lem:OptOPPM}
\State Update $(x_{t},y_{t})=\arg\min_{x\in X}\max_{y\in Y} h_t(x,y)+B_{\phi}\big(x, \widetilde{x}_{t}^\phi\big)/\eta_t-B_{\psi}\big(y, \widetilde{y}_{t}^\psi\big)/\gamma_t$
\State Output $(x_{t},y_{t})$, then observe a continuous convex-concave payoff function $f_t$
\State $x'_t=\arg\min_{x\in X}f_t\left(x, y_t\right)$, $y'_t=\arg\max_{y\in Y} f_t\left(x_t, y\right)$
\IIf{$\sum_{\tau=1}^t \big\lVert x'_\tau-x'_{\tau-1}\big\rVert> C^1$} $C^1\gets 2C^1$ \Comment{Player~1 doubles his preset}
\IIf{$\sum_{\tau=1}^t \big\lVert y'_\tau-y'_{\tau-1}\big\rVert> C^2$} $C^2\gets 2C^2$ \Comment{Player~2 doubles his preset}
\State $\widetilde{x}_{t+1}=\arg\min_{x\in X}f_t(x,y_t)+B_{\phi}\big(x, \widetilde{x}_{t}^\phi\big)/\eta_t$
\State $\widetilde{y}_{t+1}=\arg\max_{y\in Y}f_t(x_t,y)-B_{\psi}\big(y, \widetilde{y}_{t}^\psi\big)/\gamma_t$
\EndFor
\end{algorithmic}
\end{algorithm}

Utilizing \cref{lem:reg-dualgap,lem:reg-nereg}, it can be inferred that \cref{alg:OptOPPM} provides guarantees regarding both the D-Gap and the NE-Reg. 
\begin{theorem}[Performance of {\cref{alg:OptOPPM}}]
\label{thm:OptOPPM}
If two players output strategy pairs according to \cref{alg:OptOPPM}, then
\begin{equation*}
\begin{aligned}
\dualgap_T, \neregret_T 
\leq O\left(\min\left\{V'_T, \sqrt{(1+\log_2 C_T+C_T)T}\right\}\right),
\end{aligned}
\end{equation*}
where $V'_T=\sum_{t=1}^{T}\rho\left(f_{t}, h_{t}\right)$. 
\end{theorem}

\cref{alg:OptOPPM} is designed to surpass \cref{alg:OPPM}. 
By selecting $h_t=f_{t-n}$, where $n$ is a positive integer, OptOPPM ensures a sharp metric bound of $O(1)$ in environments with a periodicity of $k$, provided that $k$ is a factor of $n$. 

In \cref{alg:OptOPPM}, the update rules for $\widetilde{x}_{t+1}$ and $\widetilde{y}_{t+1}$ can be determined numerically~\citet{song2018fully}, or transformed into a closed-form expression where feasible. 

\subsection{Optimistic OPPM with Multiple Predictors}

The OptOPPM algorithm is nearly optimal and can reduce the D-Gap with an accurate predictor sequence. 
However, it currently supports only one predictor, $h_t$. 
Given the continuously changing and uncertain nature of the environment, multiple prediction models are often under consideration. 
An online algorithm that can track the best predictor among many and stay optimal even when all predictors perform poorly would be more versatile. 
This section extends OptOPPM to support multiple predictors. 

Let's consider that there are $d$ models, denoted by $M^k$, $k=1\!:\!d$. 
In round $t$, each model $M^k$ provides its predictor $h_t^k$. 
Given that OptOPPM is limited to a single predictor, a logical approach is to compute the weighted average of these $d$ predictors, formulated as $h_t=\sum_{k=1}^d \omega_t^k h_t^k$, with the weights $\omega_t$ being derived from a specific implementation of ``learning from expert advice''. 

Consider employing the ``clipped'' variant of the Hedge algorithm to generate the weight coefficients. The clipped Hedge can be formalized as follows: 
\begin{equation*}
\begin{aligned}
\omega_{t+1}=\arg\min_{\omega\in \bigtriangleup_d^\alpha} \theta_t \left\langle L_{t}, \omega\right\rangle + \mathrm{KL}(\omega,\omega_{t}),
\end{aligned}
\end{equation*}
where $\bigtriangleup_d^\alpha=\{w\in\mathbb{R}_+^d\mid \lVert w\rVert_1=1,\ w^i\geq\alpha/d,\ \forall i\in 1\!:\!d\}$, $L_t$ is a linearized loss, $\theta_t>0$ is the learning rate, and KL represents the Kullback-Leibler divergence. 

Detailed parameter settings are delineated in \cref{alg:MP-OptOPPM}, supported by the performance guarantees outlined in \cref{thm:MP-OptOPPM}. 
\begin{algorithm}[t]
\caption{Multi-Predictor Support Subroutine for OptOPPM}
\label{alg:MP-OptOPPM}
\algrenewcommand\algorithmicensure{\textbf{Initialize:}}
\begin{algorithmic}[1]
\Require All predictors satisfy \cref{ass:2-subgradient-bounded} and \cref{ass:err-bound}. 
\Ensure $T$, $\epsilon > 0$.
\For{$t=1,2,\cdots$ }
\State Receive a group of predictors $h_t^{1}, h_t^{2}, \cdots, h_t^{d}$
\State Provide predictor $h_t=\sum_{k=1}^d \omega_t^k h_t^k$ to OptOPPM
\vspace*{0.1em}\State Obtain $f_t$, $x_t$, $\widetilde{x}_{t+1}$, $y_t$, and $\widetilde{y}_{t+1}$ from \cref{alg:OptOPPM}
\IIf{$t>T$} $T\gets 2T$ \Comment{Doubling trick}
\vspace*{0.3em}\State $\displaystyle L_t=\left[\,\max\left\{\begin{aligned}&\,\big|f_t(x_t,y_t)-h_t^k(x_t,y_t)\big|,\,\\
&\,\big|f_t(\widetilde{x}_{t+1}, y_t)-h_t^k(\widetilde{x}_{t+1},y_t)\big|,\,\\
&\,\big|f_t(x_t,\widetilde{y}_{t+1})-h_t^k(x_t,\widetilde{y}_{t+1})\big|\,\end{aligned}\right\}\,\right]_{k\in 1:d}$, $\displaystyle\theta_{t}=\frac{\ln T}{\epsilon+\sum_{\tau=1}^{t-1}\sigma_{\tau}}$
\vspace*{0.1em}\State $\omega_{t+1} \!=\! \arg\min\nolimits_{\omega\in \bigtriangleup_d^{d/T}} \theta_t \left\langle L_{t}, \omega\right\rangle \!+\! \mathrm{KL}(\omega,\omega_{t})$, $\sigma_{t} \!=\! \langle L_{t}, \omega_{t}-\omega_{t+1}\rangle \!-\! \mathrm{KL}(\omega_{t+1}, \omega_{t})/\theta_{t}$
\EndFor
\end{algorithmic}
\end{algorithm}
\begin{property}
\label{ass:err-bound}
All prediction errors are bounded, that is, $\exists L_\infty<+\infty$, $\forall t$, $\forall k\in 1\!:\!d$, we have that $\left\lvert \rho\left(f_t,h_t^k\right)\right\rvert\leq L_\infty$. 
\end{property}
\begin{theorem}[Performance of {\cref{alg:MP-OptOPPM}}]
\label{thm:MP-OptOPPM}
Let there be $d$ predictors satisfying \cref{ass:err-bound}. 
If two players output strategy pairs according to \cref{alg:OptOPPM} and integrate the $d$ predictors following \cref{alg:MP-OptOPPM}, then 
\begin{equation*}
\begin{aligned}
\dualgap_T, \neregret_T 
\leq \widetilde{O}\left(\min\left\{V^1_T,\cdots,V^d_T, \sqrt{(1+C_T)T}\right\}\right),
\end{aligned}
\end{equation*}
where $V^k_T=\sum_{t=1}^{T}\rho\left(f_{t}, h^k_{t}\right)$ represents the cumulative error of the $k$-th predictor. 
\end{theorem}

\cref{alg:MP-OptOPPM} is designed to extend multi-predictor support for OptOPPM. 
We refer to the integration of \cref{alg:OptOPPM,alg:MP-OptOPPM} as \emph{OptOPPM with multiple predictors}. 

For the clipped Hedge, an efficient solution involves a slight modification to the algorithm depicted in Figure~3 of \citet{herbster2001tracking}. Refer to Algorithm~4 
in \cite{2024arXiv240704591M}. 

\section{Experiments}

This section validates the effectiveness of our algorithms through experimentation and demonstrates the cancellation phenomenon that occurs within the absolute value of NE-Reg. 
\cref{sec:osp-occo-setup} describes the experimental setup, and \cref{sec:osp-occo-results} presents the experimental results. 

\subsection{Setup}
\label{sec:osp-occo-setup}

\begin{table}[t]
\centering
\caption{Four Settings for Synthesis Experiments. In this table, $z_1(t)=\ln(1+t)$ is a logarithmic growth function of $t$, and $z_2(t)=\ln\ln(\mathrm{e}+t)$ is a log-logarithmic growth function of $t$. As time $t$ progresses, the increments of $z_1$ and $z_2$ slow down. We represent the saddle point $(x_t^*, y_t^*)$ in the complex form $x_t^*+i y_t^*$, where $i$ is the imaginary unit and satisfying the equation $i^{2}=-1$. }
\label{tab:rules}
\begin{tabular}{ccccc}
\toprule
Case & \hypertarget{ruleI}{I} & \hypertarget{ruleII}{II} & \hypertarget{ruleIII}{III} & \hypertarget{ruleIV}{IV} \\
\midrule
$x_t^*+i y_t^*$ & $z_2(t)\mathrm{e}^{i z_1(t)}$ & $z_2(t)\mathrm{e}^{i\pi t+i z_2(t)}$ & $z_2(t)\mathrm{e}^{i\frac{2\pi}{3} t+i z_2(t)}$
& $\sqrt{2}\mathrm{e}^{i\left(\frac{8\pi}{9}+\arg(x_t+i y_t)\right)}$ \\[3pt]
Trajectories & 
\begin{tabular}{c}
\includegraphics[width=0.185\textwidth]{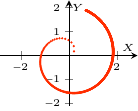}
\end{tabular} &
\begin{tabular}{c}
\includegraphics[width=0.185\textwidth]{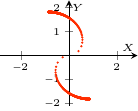}
\end{tabular} &
\begin{tabular}{c}
\includegraphics[width=0.185\textwidth]{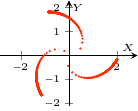}
\end{tabular} &
\begin{tabular}{c}
\includegraphics[width=0.185\textwidth]{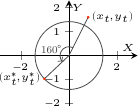}
\end{tabular} \\
Property & $\rho(f_t,f_{t-1})\rightarrow 0$ & $\rho(f_t,f_{t-2})\rightarrow 0$ & $\rho(f_t,f_{t-3})\rightarrow 0$ & Adversarial \\
\bottomrule
\end{tabular}
\end{table}

Consider the following synthesis problem: in round $t$, Players~1 and~2 jointly select a strategy pair $(x_t, y_t)\in X\times Y$, where $X=\left[-4, 4\right]$ and $Y=\left[-4, 4\right]$, and then the environment feeds back a convex-concave payoff function $f_t$:
\begin{equation*}
\begin{aligned}
f_t\left(x,y\right)=\frac{1}{2}\left(x-x_t^*\right)^2-\frac{1}{2}\left(y-y_t^*\right)^2+\left(x-x_t^*\right)\left(y-y_t^*\right), 
\end{aligned}
\end{equation*}
where $(x_t^*,y_t^*)\in X\times Y$ is the saddle point of $f_t$. 
It is evident that \cref{ass:X-Y-bounded,ass:2-subgradient-bounded} are satisfied. 
To determine the payoff function $f_t$, it suffices to fix the saddle point $(x_t^*,y_t^*)$. 
We set up four cases, which are listed in \cref{tab:rules}. 
Case~\hyperlink{ruleI}{I} is characterized by an asymptotic stability, where the saddle points of the payoff functions exhibit a decelerating trend over time. 
Cases~\hyperlink{ruleII}{II} and~\hyperlink{ruleIII}{III} display traits of periodic oscillation; the saddle points under Case~\hyperlink{ruleII}{II} vacillate between two branches, whereas those under Case~\hyperlink{ruleIII}{III} alternate among three branches. 
Case~\hyperlink{ruleIV}{IV} is indicative of an adversarial setting. As shown in its figure, upon the players' selection of the strategy pair $(x_t,y_t)$, the environment engenders the saddle point $(x_t^*,y_t^*)$ by initially rotating the strategy pair by $8\pi/9$, followed by its projection onto a circle with a radius of $\sqrt{2}$. 
This setting results in the absence of any algorithm that approximates saddle points. 

Next, we proceed to instantiate three algorithms: OPPM, OptOPPM, and OptOPPM with multiple predictors. 
Let $\phi(x)=x^2/2$ and $\psi(y)=y^2/2$. Consequently, both $B_\phi$ and $B_\psi$ are bounded and exhibit Lipschitz continuity with respect to their first variables. 
All algorithms are initialized with random initial values, as previous analysis has shown that all metric bounds are invariant to initial conditions. 
To circumvent an infinitely large initial learning rate, we set $\epsilon=0.1$. 
For the OptOPPM algorithm, we employ the predictor $h_t=f_{t-4}$, enabling it to attain a sharp metric bound of $O(1)$ in environments that are either stationary or exhibit periodicity with cycles of 2 or 4. 
For the OptOPPM with multiple predictors, we configure three predictors: $h_t^1=f_{t-4}$, $h_t^2=f_{t-5}$, and $h_t^3=f_{t-6}$, allowing it to enjoy a sharp metric bound of $\widetilde{O}(1)$ in environments that are stationary or have periodicity with cycles of 2, 3, 4, 5, or 6. 

Finally, we adopt the algorithm delineated by \citet{zhang2022noregret} as the foundational benchmark for comparison. 

\begin{figure*}[ht]
\centering\captionsetup[subfigure]{aboveskip=-0.1em,belowskip=-1.3em}
\begin{subfigure}[t]{0.48\textwidth}
\caption{Average D-Gaps in Case~\hyperlink{ruleI}{I}\\~}
\centering
\includegraphics[width=0.95\textwidth]{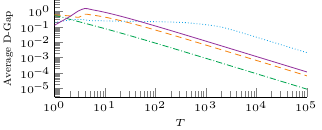}
\label{fig:rule1-dg}
\end{subfigure}
\begin{subfigure}[t]{0.48\textwidth}
\caption{Average NE-Regs in Case~\hyperlink{ruleI}{I}\\~}
\centering
\includegraphics[width=0.95\textwidth]{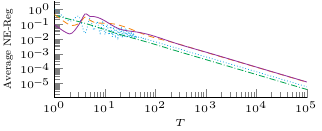}
\label{fig:rule1-ne}
\end{subfigure}
\begin{subfigure}[t]{0.48\textwidth}
\caption{Average D-Gaps in Case~\hyperlink{ruleII}{II}\\~}
\centering
\includegraphics[width=\textwidth]{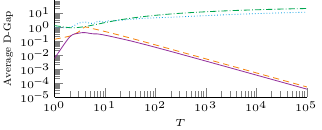}
\label{fig:rule2-dg}
\end{subfigure}
\begin{subfigure}[t]{0.48\textwidth}
\caption{Average NE-Regs in Case~\hyperlink{ruleII}{II}\\~}
\centering
\includegraphics[width=0.95\textwidth]{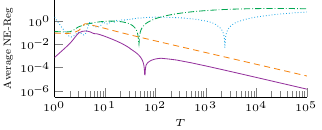}
\label{fig:rule2-ne}
\end{subfigure}
\\[-1em]
\begin{subfigure}[t]{0.48\textwidth}
\caption{Average D-Gaps in Case~\hyperlink{ruleIII}{III}\\~}
\centering
\includegraphics[width=0.95\textwidth]{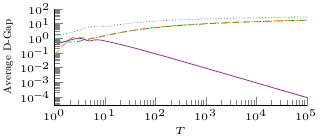}
\label{fig:rule3-dg}
\end{subfigure}
\begin{subfigure}[t]{0.48\textwidth}
\caption{Average NE-Regs in Case~\hyperlink{ruleIII}{III}\\~}
\centering
\includegraphics[width=0.95\textwidth]{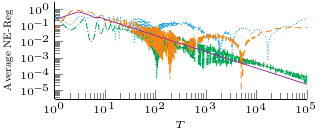}
\label{fig:rule3-ne}
\end{subfigure}
\begin{subfigure}[t]{0.48\textwidth}
\caption{Average D-Gaps in Case~\hyperlink{ruleIV}{IV}\\~}
\centering
\includegraphics[width=0.95\textwidth]{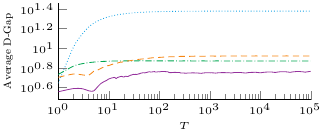}
\label{fig:rule4-dg}
\end{subfigure}
\begin{subfigure}[t]{0.48\textwidth}
\caption{Average NE-Regs in Case~\hyperlink{ruleIV}{IV}\\~}
\centering
\includegraphics[width=0.95\textwidth]{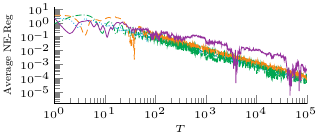}
\label{fig:rule4-ne}
\end{subfigure}
\\[-0.3em]
\newcommand{\irplotA}{\raisebox{2pt}{\tikz{\draw[Cerulean,densely dotted, thick](0,0) -- (6.9mm,0);}}}
\newcommand{\irplotB}{\raisebox{2pt}{\tikz{\draw[Green,densely dashdotted, thick](0,0) -- (6.9mm,0);}}}
\newcommand{\irplotC}{\raisebox{2pt}{\tikz{\draw[BurntOrange,densely dashed, thick](0,0) -- (6.9mm,0);}}}
\newcommand{\irplotD}{\raisebox{2pt}{\tikz{\draw[Purple,solid, thick](0,0) -- (6.9mm,0);}}}
\begin{subfigure}[t]{0.99\textwidth}
\centering\vspace*{-1em}
\sbox1{\irplotA}\sbox2{\irplotB}\sbox3{\irplotC}\sbox4{\irplotD}%
\begin{equation*}
\begin{aligned}
&\scriptsize\usebox1~\textup{The algorithm delineated by {\citet{zhang2022noregret}}} &&\quad
\scriptsize\usebox2~\textup{OPPM} \\[-3pt]
&\scriptsize\usebox3~\textup{OptOPPM} &&\quad
\scriptsize\usebox4~\textup{OptOPPM with multiple predictors}
\end{aligned}
\end{equation*}
\end{subfigure}
\caption{Average D-Gaps and Average NE-Regs of Algorithms}
\label{fig:osp}
\end{figure*}

\subsection{Results}
\label{sec:osp-occo-results}

We run the repeated game for $10^5$ rounds and record the time averaged duality gap $\frac{1}{T}\dualgap_T$ and the time averaged dynamic Nash equilibrium regret $\frac{1}{T}\neregret_T$. 
The trajectory of the average duality gap approaches zero, indicating that the players' decisions in most rounds are close to saddle points. 

All experimental outcomes align with theoretical expectations. 
In the benign environment of Case~\hyperlink{ruleI}{I}, the average D-Gap and NE-Reg trajectories of all algorithms converge (see \cref{fig:rule1-dg} and \cref{fig:rule1-ne}). 
In the periodically oscillatory environment of Case~\hyperlink{ruleII}{II}, the average D-Gap and NE-Reg trajectories of OptOPPM and OptOPPM with multiple predictors gradually converge (see \cref{fig:rule2-dg} and \cref{fig:rule2-ne}). 
In the oscillatory scenario of Case~\hyperlink{ruleIII}{III}, the average D-Gap and NE-Reg trajectories of OptOPPM with multiple predictors converge (see \cref{fig:rule3-dg} and \cref{fig:rule3-ne}). 
Notably, the average NE-Reg trajectory of OPPM converges amidst pronounced fluctuations, indicating internal cancellation within the absolute value of NE-Reg (see \cref{fig:rule3-ne}). 
In the adversarial setting of Case~\hyperlink{ruleIV}{IV}, no algorithm approximates the saddle points, resulting in non-converging average D-Gap trajectories (see \cref{fig:rule4-dg}). However, the average NE-Reg trajectories exhibit intense fluctuations, indicating internal cancellation within the absolute value of NE-Reg (see \cref{fig:rule4-ne}). 
The results show that OPPM and its variants outperform the algorithm by \citet{zhang2022noregret} in D-Gap performance. Notably, OptOPPM with multiple predictors consistently approximates saddle points in Cases~\hyperlink{ruleI}{I}, \hyperlink{ruleII}{II}, and \hyperlink{ruleIII}{III}.

Given the insights garnered from both theoretical analysis and empirical results, it seems prudent to reconsider the reliance on NE-Reg as a metric.

\section{Conclusion}

This study addresses the online saddle point problem by introducing three adaptations of the proximal point method: OPPM, OptOPPM, and OptOPPM with multiple predictors. 
These methods are crafted to secure upper bounds on D-Gap and NE-Reg, ensuring near-optimal performance in relation to D-Gap. 
In favorable conditions, such as stationary payoff functions, they maintain near-constant metric bounds. The study also questions the reliability of NE-Reg as a metric and validates the algorithms' effectiveness through experiments.

\subsubsection{Acknowledgement}
This preprint has no post-submission improvements or corrections. The Version of Record of this contribution is published in the Neural Information Processing, ICONIP 2024 Proceedings and is available online at \url{https://doi.org/10.1007/978-981-96-6579-2\_27}. 
%
%
%
\bibliographystyle{splncs04}
\bibliography{reference}

\clearpage
\appendix
\onecolumn
\begin{center}{
\Large\textbf{Technical Appendix}
\par}
\end{center}
\small
\section{Proof of \cref{thm:lower-bound}} 

\begin{proof}
Let $\mathcal{F}$ be the set of all convex-concave functions satisfying Assumption~A2, and let $\mathcal{L}_X^{G_X}=\big\{\ell\,\big|\,\ell(x)+0y\in\mathcal{F}\big\}$, let $\mathcal{L}_Y^{G_X}=\big\{\ell\,\big|\,0x-\ell(y)\in\mathcal{F}\big\}$. 
Formally, we need to prove that
\begin{equation*}
\begin{aligned}
\exists f_{1:T}\in\mathcal{F},\quad\text{such that}\quad\sum_{t=1}^T\max_{y\in Y} f_t\left(x_t, y\right)-\sum_{t=1}^T\min_{x\in X}f_t\left(x, y_t\right)\geq\symit{\Omega}\big(\sqrt{(1+C_T)T}\big). 
\end{aligned}
\end{equation*}
Indeed, we may choose $f_t(x,y)=\alpha_t(x)-\beta_t(y)\in\mathcal{F}$, where $\alpha_t$ and $\beta_t$ are both convex functions. 
$\forall C\in[0,T(D_X+D_Y)]$ and $\forall P\in[0,C]$, $\exists\alpha_{1:T}$ and $\exists \beta_{1:T}$, such that 
\begin{equation}
\label{eqpf:lb}
\begin{aligned}
\sum_{t=1}^T\max_{y\in Y}\ & f_t\left(x_t, y\right)-\sum_{t=1}^T\min_{x\in X}f_t\left(x, y_t\right)\\
=\ &\max_{\forall u_{1:T}\in X,\forall v_{1:T}\in Y}\sum_{t=1}^T \Bigl(f_t\left(x_t, v_t\right)- f_t\left(u_t, y_t\right)\Bigr) \\
\geq\ &\max_{\sum_{t=1}^T\left\lVert u_t-u_{t-1}\right\rVert\leq P,\, \sum_{t=1}^T\left\lVert v_t-v_{t-1}\right\rVert\leq C-P}\sum_{t=1}^T \Big(f_t\left(x_t,v_t\right)- f_t\left(u_t,y_t\right)\Big) \\
=\ &\max_{\sum_{t=1}^T\left\lVert u_t-u_{t-1}\right\rVert\leq P}\sum_{t=1}^T \Big(\alpha_t\left(x_t\right)- \alpha_t\left(u_t\right)\Big) \\
&+\max_{\sum_{t=1}^T\left\lVert v_t-v_{t-1}\right\rVert\leq C-P}\sum_{t=1}^T \Big(\beta_t\left(y_t\right)- \beta_t\left(v_t\right)\Big) \\
\geq\ &\symit{\Omega}\left(\sqrt{(1+P)T}\right)+\symit{\Omega}\left(\sqrt{(1+C-P)T}\right) \\
=\ &\symit{\Omega}\left(\sqrt{(1+C)T}\right),
\end{aligned}
\end{equation}
where the last ``$\leq$'' follows from \cref{lem:lower-bound}. 
Note that \cref{eqpf:lb} holds for arbitrary $C\in[0,T(D_X+D_Y)]$. 
Choosing $C=C_T$ yields the desired result. 
\end{proof}

\begin{lemma}[Theorem~2 of {\citet{zhang2018adaptive}}]
\label{lem:lower-bound}
In the context of Online Convex Optimization, let the feasible set $X$ be compact and convex, and let $\mathcal{L}_X^G$ be the set of all convex loss functions defined on $X$ with subgradients bounded by $G$. 
Regardless of the strategy adopted by the player, there always exists a comparator sequence $u_{1:T}$ satisfying $\sum_{t=1}^T\left\lVert u_t-u_{t-1}\right\rVert\leq P$, and a sequence of loss functions $\ell_{1:T}\in\mathcal{L}_X^G$, ensuring that the regret is not less than $\symit{\Omega}\big(\sqrt{(1+P)T}\big)$.
\end{lemma}

\section{Proof of \cref{lem:OPPM}}
\label{app:Analysis-OPPM}

\begin{proof}
The first-order optimality condition of OPPM implies that 
\begin{equation*}
\begin{aligned}
&\exists\nabla_x f_t(x_{t+1},y_{t+1}),&&\forall x'\in X,&&
\textstyle\big\langle \eta_t \nabla_x f_t(x_{t+1},y_{t+1})+x_{t+1}^\phi-x_{t}^\phi, x_{t+1}-x'\big\rangle\leq 0, \\
&\exists\nabla_y (-f_t)(x_{t+1},y_{t+1}),&&\forall y'\in Y,&&
\textstyle\big\langle \gamma_t \nabla_y (-f_t)(x_{t+1},y_{t+1})+y_{t+1}^\psi-y_{t}^\psi, y_{t+1}-y'\big\rangle\leq 0.
\end{aligned}
\end{equation*}
Let's take Player~$1$ as an example.
The identity transformation on the instantaneous individual regret is as follows: 
\begin{equation}
\label{eqpf:iomda-individual-regret-decomposition}
\begin{aligned}
f_t(x_t,y_t)-f_t(x'_{t}, y_t)
=\ &f_t(x_t,y_t)-f_t(x_{t+1}, y_{t+1})+f_t(x'_{t+1}, y_{t+1})-f_t(x'_{t}, y_t) \\
&+\underbrace{f_t(x_{t+1}, y_{t+1})-f_t(x'_{t+1}, y_{t+1})}_{\term{a}}. 
\end{aligned}
\end{equation}
Applying convexity and first-order optimality condition, we get
\begin{equation*}
\begin{aligned}
\termref{eqpf:iomda-individual-regret-decomposition}{a}
&\leq\left\langle\nabla_x f_t(x_{t+1}, y_{t+1}), x_{t+1}-x'_{t+1}\right\rangle 
\leq\frac{1}{\eta_t }\left\langle x_t^\phi-x_{t+1}^\phi, x_{t+1}-x'_{t+1}\right\rangle \\
&=\frac{1}{\eta_t }\left[B_{\phi}\big(x'_{t+1}, x_{t}^\phi\big)-B_{\phi}\big(x'_{t+1}, x_{t+1}^\phi\big) -B_{\phi}\big(x_{t+1}, x_{t}^\phi\big)\right].
\end{aligned}
\end{equation*}
Summing \cref{eqpf:iomda-individual-regret-decomposition} over time yields
\begin{equation}
\label{eqpf:iomda-individual-regret}
\begin{aligned}
\regret^1_T
\leq\ &
\underbrace{\sum_{t=1}^{T}\frac{1}{\eta_t}\Bigl(B_{\phi}\big(x'_{t+1}, x_{t}^\phi\big)-B_{\phi}\big(x'_{t+1}, x_{t+1}^\phi\big)\Bigr)}_{\term{a}}\\
&+\underbrace{\sum_{t=1}^{T}\Bigl(f_t(x_t,y_t)-f_t(x_{t+1}, y_{t+1})+f_t(x'_{t+1}, y_{t+1})-f_t(x'_{t}, y_t)\Bigr)}_{\term{b}}.
\end{aligned}
\end{equation}
Since the learning rate $\eta_t$ does not increase, $B_\phi$ is upper bounded by $L_\phi D_X$ and is $L_\phi$-Lipschitz for the first variable, we have that 
\begin{equation*}
\begin{aligned}
&\!\!\termref{eqpf:iomda-individual-regret}{a} \\
&\leq \sum_{t=1}^{T}\frac{1}{\eta_t}\left(B_{\phi}\big(x'_{t+1}, x_{t}^\phi\big)-B_{\phi}\big(x'_{t}, x_{t}^\phi\big)\right) + \frac{B_{\phi}\big(x'_{1}, x_{1}^\phi\big)}{\eta_0} + \sum_{t=1}^{T}\Big(\frac{1}{\eta_{t}}-\frac{1}{\eta_{t-1}}\Big)B_{\phi}\big(x'_{t}, x_{t}^\phi\big) \\
&\leq\frac{L_\phi D_X}{\eta_T}+\sum_{t=1}^{T}\frac{L_\phi}{\eta_t}\left\lVert x'_{t+1}-x'_{t}\right\rVert
\leq \frac{2 L_\phi D_X}{\eta_T}+\sum_{t=1}^{T}\frac{L_\phi}{\eta_t}\left\lVert x'_{t}-x'_{t-1}\right\rVert.
\end{aligned}
\end{equation*}
Let $\symit{\Sigma}_T^1=\left(\termref{eqpf:iomda-individual-regret}{b}\right)_+$, then \cref{eqpf:iomda-individual-regret} can be rearranged as follows: 
\begin{equation*}
\begin{aligned}
\regret^1_T\leq L_\phi\left(2\frac{D_X}{\eta_T}+\sum_{t=1}^{T}\frac{1}{\eta_t}\left\lVert x'_{t}-x'_{t-1}\right\rVert\right)+\symit{\Sigma}_T^1.
\end{aligned}
\end{equation*}
Likewise, 
\begin{equation*}
\begin{aligned}
\regret^2_T\leq L_\psi\left(2\frac{D_Y}{\gamma_T}+\sum_{t=1}^{T}\frac{1}{\gamma_t}\left\lVert y'_t-y'_{t-1}\right\rVert\right)+\symit{\Sigma}_T^2,
\end{aligned}
\end{equation*}
where 
\begin{equation*}
\begin{aligned}
\symit{\Sigma}_T^2=\left(\sum_{t=1}^{T}\Bigl(f_t(x_t,y'_t)-f_t(x_{t+1},y'_{t+1})+f_t(x_{t+1},y_{t+1})-f_t(x_t,y_t)\Bigr)\right)_+.
\end{aligned}
\end{equation*}
Let $L=\max\{L_\phi,L_\psi\}$, $D=\max\{D_X,D_Y\}$, $C_T=\sum_{t=1}^T \left(\left\lVert x'_{t}-x'_{t-1}\right\rVert + \left\lVert y'_t-y'_{t-1}\right\rVert\right)$, and let learning rates satisfy $\eta_t=\gamma_t$. 
Consequently, the individual regrets are subject to the following public upper bound: 
\begin{equation*}
\begin{aligned}
\regret^1_T, \regret^2_T \leq L\frac{2D+C_T}{\eta_T}+\max\left\{\symit{\Sigma}_T^1,\symit{\Sigma}_T^2\right\}.
\end{aligned}
\end{equation*}
Let $\symit{\Sigma}_T = \max\left\{\symit{\Sigma}_T^1,\symit{\Sigma}_T^2\right\}$, $\Delta_t=(\symit{\Sigma}_t-\max_{\tau\in 1:t-1}\symit{\Sigma}_{\tau})_+$. 
We claim that $\sum_{\tau=1}^t \Delta_\tau = \max_{\tau\in 1:t}\symit{\Sigma}_{\tau}$. 
This claim can be proved by induction. 
It is obvious that $\Delta_1=\symit{\Sigma}_1$. 
Now we assume the claim holds for $t-1$ and prove it for $t$: 
\begin{equation*}
\begin{aligned}
\sum_{\tau=1}^t \Delta_\tau
=\Delta_t + \sum_{\tau=1}^{t-1} \Delta_\tau
&=\Big(\symit{\Sigma}_t-\max_{\tau\in 1:t-1}\symit{\Sigma}_{\tau}\Big)_+ + \max_{\tau\in 1:t-1}\symit{\Sigma}_{\tau} \\
&=\begin{cases}
\symit{\Sigma}_t & \symit{\Sigma}_t\geq\max_{\tau\in 1:t-1}\symit{\Sigma}_{\tau} \\
\max_{\tau\in 1:t-1}\symit{\Sigma}_{\tau} & \symit{\Sigma}_t<\max_{\tau\in 1:t-1}\symit{\Sigma}_{\tau}
\end{cases} 
= \max_{\tau\in 1:t}\symit{\Sigma}_{\tau}.
\end{aligned}
\end{equation*}
Let $C$ be a preset upper bound of $C_T$, then applying the prescribed learning rate yields 
\begin{equation}
\label{eqpf:OPPM-bound-all}
\begin{aligned}
\regret^1_T, \regret^2_T \leq \epsilon+2\sum_{t=1}^{T} \Delta_t
= \epsilon+2\max_{t\in 1:T}\symit{\Sigma}_{t}. 
\end{aligned}
\end{equation}
Next, we estimate the upper bound of the r.h.s. of the above inequality in two ways. 
The firse one: 
\begin{equation}
\label{eqpf:OPPM-bound1}
\begin{aligned}
\max_{t\in 1:T}\symit{\Sigma}_{t}\leq O\left(1\right)+2\sum_{t=1}^{T}\rho\left( f_{t},f_{t-1}\right). 
\end{aligned}
\end{equation}
Indeed, 
\begin{equation*}
\begin{aligned}
\symit{\Sigma}_{t}^1
=\ &\bigg(\sum_{\tau=1}^{t}\Big(f_\tau(x_\tau,y_\tau)-f_\tau(x_{\tau+1}, y_{\tau+1})+f_\tau(x'_{\tau+1}, y_{\tau+1})-f_\tau(x'_\tau, y_\tau)\Big)\bigg)_+ \\
\leq\ &\Big|f_{0}(x_1,y_1)-f_{0}(x'_1,y_1)-f_{t}(x_{t+1}, y_{t+1})+f_{t}(x'_{t+1}, y_{t+1})\Big| \\
&+ \sum_{\tau=1}^{t}\Big|f_{\tau}(x_\tau,y_\tau)-f_{\tau-1}(x_{\tau}, y_{\tau})\Big|+\sum_{\tau=1}^{t}\Big|f_{\tau}(x'_\tau,y_\tau)-f_{\tau-1}(x'_{\tau}, y_{\tau})\Big| \\
\leq\ & 2\left(D_X G_X+D_Y G_Y\right) + 2\sum_{\tau=1}^{t}\rho\left( f_{\tau},f_{\tau-1}\right),
\end{aligned}
\end{equation*}
where the last ``$\leq$'' follows from \cref{cor:payoff-bound}. 
Likewise, 
\begin{equation*}
\begin{aligned}
\symit{\Sigma}_{t}^2
\leq 2\left(D_X G_X+D_Y G_Y\right) + 2\sum_{\tau=1}^{t}\rho\left( f_{\tau},f_{\tau-1}\right).
\end{aligned}
\end{equation*}
So \cref{eqpf:OPPM-bound1} holds. 
The other one: 
\begin{equation}
\label{eqpf:OPPM-bound2}
\begin{aligned}
\sum_{t=1}^{T}\Delta_t\leq O\left(\sqrt{(1+C)T}\right). 
\end{aligned}
\end{equation}
The derivation is as follows. 
According to \cref{cor:OPPM-stability}, $\left\lVert x_t-x_{t+1}\right\rVert\leq\min\{D_X, \eta_t G_X\}$, and $\left\lVert y_t-y_{t+1}\right\rVert\leq\min\{D_Y, \eta_t G_Y\}$. 
Let $G=\max\{G_X,G_Y\}$. 
Thus we have that 
\begin{equation*}
\begin{aligned}
\bigl(\symit{\Sigma}_{t}^1-\symit{\Sigma}_{t-1}^1\bigr)_+\!\!
\leq\ &\big(f_t(x_t,y_t)-f_t(x_{t+1}, y_{t+1})+f_t(x'_{t+1}, y_{t+1})-f_t(x'_{t}, y_t)\big)_+ \\
\leq\ &\big|f_t(x_t,y_t)-f_t(x_{t}, y_{t+1})\big| 
+\big|f_t(x_{t}, y_{t+1})-f_t(x_{t+1}, y_{t+1})\big| \\
&+\big|f_t(x'_{t+1}, y_{t+1})-f_t(x'_{t+1}, y_t)\big|
+\big|f_t(x'_{t+1}, y_t)-f_t(x'_{t}, y_t)\big| \\
\leq\ & G_Y \left\lVert y_t-y_{t+1}\right\rVert+G_X \left\lVert x_t-x_{t+1}\right\rVert + G_Y \left\lVert y_t-y_{t+1}\right\rVert + G_X \left\lVert x'_t-x'_{t+1}\right\rVert \\
\leq\ & \min\left\{ D_X G_X + 2D_Y G_Y,\eta_t \left(G_X^2 + 2G_Y^2 \right)\right\} + G_X \left\lVert x'_t-x'_{t+1}\right\rVert \\
\leq\ & \min\left\{ 3D G , 3\eta_t G^2 \right\} + G \left(\left\lVert x'_t-x'_{t+1}\right\rVert + \left\lVert y'_t-y'_{t+1}\right\rVert\right).
\end{aligned}
\end{equation*}
Likewise, 
\begin{equation*}
\begin{aligned}
\bigl(\symit{\Sigma}_{t}^2-\symit{\Sigma}_{t-1}^2\bigr)_+
\leq \min\left\{ 3D G , 3\eta_t G^2 \right\} + G \left(\left\lVert x'_t-x'_{t+1}\right\rVert + \left\lVert y'_t-y'_{t+1}\right\rVert\right).
\end{aligned}
\end{equation*}
Note that $\symit{\Sigma}_t-\symit{\Sigma}_{t-1}=\max\left\{\symit{\Sigma}_{t}^1-\symit{\Sigma}_{t-1}, \symit{\Sigma}_{t}^2-\symit{\Sigma}_{t-1}\right\}\leq\max\left\{\symit{\Sigma}_{t}^1-\symit{\Sigma}_{t-1}^1, \symit{\Sigma}_{t}^2-\symit{\Sigma}_{t-1}^2\right\}$, so we get
\begin{equation*}
\begin{aligned}
\Delta_t
=\Big(\symit{\Sigma}_t-\max_{\tau\in 1:t-1}\symit{\Sigma}_{\tau}\Big)_+
&\leq(\symit{\Sigma}_t-\symit{\Sigma}_{t-1})_+ \\
&\leq \min\left\{ 3D G , 3\eta_t G^2 \right\} + G \left(\left\lVert x'_t-x'_{t+1}\right\rVert + \left\lVert y'_t-y'_{t+1}\right\rVert\right).
\end{aligned}
\end{equation*}
Let $\xi_t=\big(\Delta_t-G \big(\left\lVert x'_t-x'_{t+1}\right\rVert+\left\lVert v_t-v_{t+1}\right\rVert\big)\big)_+$, then $\xi_t\leq\min\left\{ 3D G , 3\eta_t G^2 \right\}$, and thus, 
\begin{equation*}
\begin{aligned}
\left(\sum_{t=1}^{T-1}\xi_t\right)^2
&=\sum_{t=1}^{T-1}\xi_t^2+2\sum_{t=1}^{T-1}\xi_t\sum_{\tau=1}^{t-2}\xi_\tau
\leq\sum_{t=1}^{T-1}\xi_t^2+2\sum_{t=1}^{T-1}\xi_t\sum_{\tau=1}^{t-2}\Delta_\tau \\
&=\sum_{t=1}^{T-1}\xi_t^2+2\sum_{t=1}^{T-1}\xi_t \left(\frac{L(2D+C)}{\eta_t}-\epsilon\right) 
\leq \sum_{t=1}^{T-1}9 D^2 G^2+\sum_{t=1}^{T-1} 6G^2L(2D+C),
\end{aligned}
\end{equation*}
where the last ``$\leq$'' uses the first and second terms in the minimum of the bound for $\xi_t$ in turn. 
Now we get 
\begin{equation*}
\begin{aligned}
\sum_{t=1}^{T}\Delta_t\leq O(C_T)+\sum_{t=1}^{T}\xi_t\leq O(C_T)+O\left(\sqrt{(1+C)T}\right)\leq O\left(\sqrt{(1+C)T}\right).
\end{aligned}
\end{equation*}
So \cref{eqpf:OPPM-bound2} holds. 
Substituting \cref{eqpf:OPPM-bound1,eqpf:OPPM-bound2} into \cref{eqpf:OPPM-bound-all} gives the following conclusion: 
\begin{equation*}
\regret^1_T, \regret^2_T 
\leq O\left(\min\bigg\{\sum_{t=1}^{T}\rho\left(f_{t}, f_{t-1}\right), \sqrt{(1+C)T}\bigg\}\right).
\end{equation*}
\end{proof}

\begin{lemma}
\label{cor:payoff-bound}
Assumptions~A1 and~A2 guarantee the finiteness of the instantaneous duality gap.
\end{lemma}

\begin{proof}[Proof of \cref{cor:payoff-bound}]
$\forall x, x'\in X$, $\forall y, y'\in Y$, 
\begin{equation*}
\begin{aligned}
\big\langle\partial_x f_t\left(x',y\right), x-x'\big\rangle
&\leq f_t(x,y)-f_t(x',y)
\leq \big\langle\partial_x f_t\left(x,y\right), x-x'\big\rangle, \\ 
\big\langle\partial_y (-f_t)\left(x,y'\right), y-y'\big\rangle
&\leq f_t(x,y')-f_t(x,y)\leq \big\langle\partial_y (-f_t)\left(x,y\right), y-y'\big\rangle,
\end{aligned}
\end{equation*}
which implies that $\left\lvert f_t(x,y)-f_t(x',y)\right\rvert\leq D_X G_X$, 
$\left\lvert f_t(x,y')-f_t(x,y)\right\rvert\leq D_Y G_Y$, and 
$\lvert f_t(x,y')-f_t(x',y)\rvert
\leq\left\lvert f_t(x,y')-f_t(x,y)\right\rvert+\left\lvert f_t(x,y)-f_t(x',y)\right\rvert
\leq D_X G_X+D_Y G_Y$. 
\end{proof}

\begin{lemma}
\label{cor:OPPM-stability}
Under Assumption~A2 and \cref{pro:2-Bregman-Lipschitz}, OPPM guarantees that $\left\lVert x_t-x_{t+1}\right\rVert\leq\eta_t G_X$ and $\left\lVert y_t-y_{t+1}\right\rVert\leq\gamma_t G_Y$. 
\end{lemma}

\begin{proof}[Proof of \cref{cor:OPPM-stability}]
Let $F_t(x,y)=f_t(x,y)+B_{\phi}\big(x, x_t^\phi\big)/\eta_t-B_{\psi}\big(y, y_t^\psi\big)/\gamma_t$. Let $\symit{\Phi}=F_t(\,\cdot\,, y_{t+1})$. 
Note that $\symit{\Phi}$ is $\eta_t^{-1}$-strongly convex, so we have that
\begin{equation*}
\begin{aligned}
B_{\symit{\Phi}}\big(x_{t+1}, x_{t}^{\symit{\Phi}}\big)\geq \frac{1}{2\eta_t}\left\lVert x_t-x_{t+1}\right\rVert^2, \qquad
B_{\symit{\Phi}}\big(x_{t}, x_{t+1}^{\symit{\Phi}}\big)\geq \frac{1}{2\eta_t}\left\lVert x_t-x_{t+1}\right\rVert^2.
\end{aligned}
\end{equation*}
Choosing $x_{t+1}^{\symit{\Phi}}=0$, and adding the above two inequalities yields
\begin{equation*}
\begin{aligned}
\frac{1}{\eta_t}\left\lVert x_t-x_{t+1}\right\rVert^2
\leq\left\langle x_{t}^{\symit{\Phi}}-x_{t+1}^{\symit{\Phi}}, x_{t}-x_{t+1}\right\rangle
\leq\left\lVert x_{t}^{\symit{\Phi}}\right\rVert\left\lVert x_t-x_{t+1}\right\rVert.
\end{aligned}
\end{equation*}
Note that $\left\lVert x_{t}^{\symit{\Phi}}\right\rVert=\left\lVert \nabla_x f_t(x_{t},y_{t+1}) \right\rVert\leq G_X$, Thus we have that $\left\lVert x_t-x_{t+1}\right\rVert\leq\eta_t G_X$. 
Likewise, $\left\lVert y_t-y_{t+1}\right\rVert\leq\gamma_t G_Y$. 
\end{proof}

\section{Proof of \cref{thm:OPPM}}

\begin{proof}
Consider the \cref{alg:OPPM} initially at stage 0. 
Each time the value $C$ is doubled, the algorithm advances to the next stage. 
We use $C_n$ to represent the value of $C$ at stage $n$, and let $S_n$ be the set of indices for all rounds in the $n$-th stage. 
Assume that the game pauses after the $T$-th round is completed, at which point the OPPM is at stage $s$. 
Applying \cref{lem:OPPM} yields 
\begin{equation*}
\begin{aligned}
\regret^i_T &= \sum_{n=1}^s\regret^i_{S_n}
\leq \sum_{n=1}^s O\left(\min\bigg\{\sum_{t\in S_n}\rho\left(f_{t}, f_{t-1}\right), \sqrt{(1+2^n C_0)\left\lvert S_n\right\rvert}\bigg\}\right)  \\
&\leq O\left(\min\left\{\sum_{t=1}^T\rho\left(f_{t}, f_{t-1}\right), \sqrt{s+\sum_{n=1}^s 2^n C_0}\sqrt{\sum_{n=1}^s\left\lvert S_n\right\rvert}\right\}\right), \qquad\forall i = 1,2.
\end{aligned}
\end{equation*}
Note that $2^{s-1}C_0=C_{s-1}<C_T\leq C_s=2^s C_0$, and $T=\sum_{n=1}^s\left\lvert S_n\right\rvert$. We have that 
\begin{equation*}
\begin{aligned}
\regret^i_T \leq O\left(\min\biggl\{\sum_{t=1}^T\rho\left(f_{t}, f_{t-1}\right), \sqrt{\left(1+\log_2 C_T+4C_T\right)T}\biggr\}\right), \qquad\forall i = 1,2.
\end{aligned}
\end{equation*}
Applying \cref{lem:reg-dualgap,lem:reg-nereg} yields the conclusion to be proved. 
\end{proof}

\section{Proof of \cref{lem:OptOPPM}}
\label{app:Analysis-OptOPPM}

\begin{proof}
The first-order optimality condition of OptOPPM implies that 
\begin{equation*}
\begin{aligned}
&\exists\nabla_x h_t\big(x_t, y_t\big),&&\forall x'\in X,&&
\textstyle\bigl\langle \eta_t \nabla_x h_t\big(x_t, y_t\big) + x_t^{\phi}-\widetilde{x}_t^{\phi}, x_t-x'\bigr\rangle \leq 0, \\
&\exists\nabla_y (-h_t)\big(x_t, y_t\big),&&\forall y'\in Y,&&
\textstyle\bigl\langle \gamma_t \nabla_y (-h_t)\big(x_t, y_t\big) + y_t^{\psi}-\widetilde{y}_t^{\psi}, y_t-y'\bigr\rangle \leq 0, \\
&\exists\nabla_x f_t(\widetilde{x}_{t+1}, y_t),&&\forall x'\in X,&&
\bigl\langle \eta_t \nabla_x f_t(\widetilde{x}_{t+1}, y_t) + \widetilde{x}_{t+1}^{\phi}-\widetilde{x}_t^{\phi}, \widetilde{x}_{t+1}-x'\bigr\rangle \leq 0, \\
&\exists\nabla_y (-f_t)(x_{t},\widetilde{y}_{t+1}),&&\forall y'\in Y,&&
\textstyle\big\langle \gamma_t \nabla_y (-f_t)(x_{t},\widetilde{y}_{t+1})+\widetilde{y}_{t+1}^\psi-\widetilde{y}_{t}^\psi, \widetilde{y}_{t+1}-y'\big\rangle\leq 0.
\end{aligned}
\end{equation*}
Let's take Player~$1$ as an example.
We first perform identity transformation on the instantaneous individual regret: 
\begin{equation}
\label{eqpf:identical-trans}
\begin{aligned}
f_t(x_t, y_t)-f_t(x'_t, y_t)=\ &
\underbrace{f_t(x_t,y_t)-h_t(x_t,y_t)
+h_t(\widetilde{x}_{t+1},y_t)-f_t(\widetilde{x}_{t+1}, y_t)}_{\term{a}} \\
&+\underbrace{h_t(x_t,y_t)-h_t(\widetilde{x}_{t+1},y_t)
+f_t(\widetilde{x}_{t+1}, y_t)-f_t(x'_t, y_t)}_{\term{b}}.
\end{aligned}
\end{equation}
By using convexity and first-order optimality conditions, we get
\begin{equation}
\label{eq:1-order-x-bound}
\begin{aligned}
\termref{eqpf:identical-trans}{b}
\leq\ &\big\langle \nabla_x h_t\big(x_t, y_t\big), x_t-\widetilde{x}_{t+1}\big\rangle + \big\langle\nabla_x f_t(\widetilde{x}_{t+1}, y_t), \widetilde{x}_{t+1}-x'_t\big\rangle \\
\leq\ &\big\langle \widetilde{x}_{t}^\phi-x_{t}^\phi,x_t-\widetilde{x}_{t+1}\big\rangle/ \eta_t  + \big\langle \widetilde{x}_{t}^\phi-\widetilde{x}_{t+1}^\phi,\widetilde{x}_{t+1}-x'_t\big\rangle / \eta_t\\
=\ &\big[B_{\phi}\big(\widetilde{x}_{t+1}, \widetilde{x}_t^\phi\big)-B_{\phi}\big(\widetilde{x}_{t+1}, x_t^\phi\big)-B_{\phi}\big(x_t, \widetilde{x}_t^\phi\big)\big]/ \eta_t \\
&+\underbrace{\big[B_{\phi}\big(x'_t, \widetilde{x}_{t}^\phi\big)-B_{\phi}\big(x'_t, \widetilde{x}_{t+1}^\phi\big)\big] / \eta_t}_{\eqcolon\Phi_t}-B_{\phi}\big(\widetilde{x}_{t+1}, \widetilde{x}_{t}^\phi\big) / \eta_t.
\end{aligned}
\end{equation}
Let $\delta_t^1=\termref{eqpf:identical-trans}{a}-B_{\phi}\big(\widetilde{x}_{t+1}, x_t^\phi\big)/\eta_t$, so we have that $f_t(x_t, y_t)-f_t(x'_t, y_t)\leq\Phi_t+\delta_t^1$. 
Note that $\eta_t$ is non-increasing over time, $B_\phi$ is $L_\phi$-Lipschitz w.r.t. the first variable, and $L_\phi D_X$ is the supremum of $B_\phi$. Thus, 
\begin{equation*}
\begin{aligned}
\sum_{t=1}^{T}\Phi_t 
&\leq\frac{B_{\phi}\big(x'_{0}, \widetilde{x}_{1}^\phi\big)}{\eta_0} +\sum_{t=1}^{T}\frac{1}{\eta_t}\left(B_{\phi}\big(x'_t, \widetilde{x}_{t}^\phi\big)-B_{\phi}\big(x'_{t-1}, \widetilde{x}_{t}^\phi\big)\right) +\sum_{t=1}^{T}\left(\frac{1}{\eta_{t}}-\frac{1}{\eta_{t-1}}\right)B_{\phi}\big(x'_{t-1}, \widetilde{x}_{t}^\phi\big) \\
&\leq\frac{L_\phi D_X}{\eta_T}+\sum_{t=1}^{T}\frac{L_\phi}{\eta_t}\left\lVert x'_t-x'_{t-1}\right\rVert,
\end{aligned}
\end{equation*}
Now we get
\begin{equation*}
\begin{aligned}
\regret^1_T
\leq \frac{L_\phi D_X}{\eta_T}+\sum_{t=1}^{T}\frac{L_\phi}{\eta_t}\left\lVert x'_t-x'_{t-1}\right\rVert + \sum_{t=1}^{T} \delta_t^1.
\end{aligned}
\end{equation*}
Likewise, 
\begin{equation*}
\begin{aligned}
\regret^2_T
\leq \frac{L_\psi D_Y}{\gamma_T}+\sum_{t=1}^{T}\frac{L_\psi}{\gamma_t}\left\lVert y'_t-y'_{t-1}\right\rVert + \sum_{t=1}^{T} \delta_t^2,
\end{aligned}
\end{equation*}
where $\delta_t^2=f_t(x_t,\widetilde{y}_{t+1})-h_t(x_t,\widetilde{y}_{t+1})+h_t(x_t,y_t)-f_t(x_t,y_t)-B_{\psi}\big(\widetilde{y}_{t+1}, y_t^\psi\big)/\gamma_t$. 
Next, we verify $\delta_t^1, \delta_t^2\geq 0$. 
To verify $\delta_t^1\geq 0$, it suffices to combine the following two inequalities:
\begin{equation*}
\begin{aligned}
f_t(x_t,y_t)+B_{\phi}\big(x_t, \widetilde{x}_{t}^\phi\big)/\eta_t
&\geq f_t(\widetilde{x}_{t+1},y_t)+B_{\phi}\big(\widetilde{x}_{t+1}, \widetilde{x}_{t}^\phi\big)/\eta_t, \\
-h_t(x_t,y_t)+h_t(\widetilde{x}_{t+1},y_t)
&\geq-\big[B_{\phi}\big(\widetilde{x}_{t+1}, \widetilde{x}_t^\phi\big)-B_{\phi}\big(\widetilde{x}_{t+1}, x_t^\phi\big)-B_{\phi}\big(x_t, \widetilde{x}_t^\phi\big)\big]/ \eta_t.
\end{aligned}
\end{equation*}
The first inequality takes advantage of the optimality condition, and the second inequality is part of \cref{eq:1-order-x-bound}. 
Likewise, $\delta_t^2\geq 0$. 
So all learning rates are non-increasing. 
Let's go back to the focus on Player~1. 
The prescribed learning rate guarantees that
\begin{equation*}
\begin{aligned}
\regret^1_T
\leq \frac{L_\phi}{\eta_T}\left(D_X+C_T^1\right) + \sum_{t=1}^{T} \delta_t^1\leq\epsilon+2\sum_{t=1}^{T}\delta_t^1.
\end{aligned}
\end{equation*}
On the one hand, $\delta_t^1\leq 2\rho(f_t,h_t)$ causes
\begin{equation}
\label{eqpf:OptIOMDA-x-bound-part1}
\begin{aligned}
\regret^1_T\leq\epsilon+4\sum_{t=1}^{T}\rho(f_t,h_t).
\end{aligned}
\end{equation}
On the other hand, notice that 
\begin{equation}
\label{eqpf:gradient-bound}
\begin{aligned}
\delta_t^1
&\leq \big\langle \nabla_x f_t(x_t, y_t)-\nabla_x h_t(\widetilde{x}_{t+1},y_t), x_t-\widetilde{x}_{t+1}\big\rangle-B_{\phi}\big(\widetilde{x}_{t+1}, x_t^\phi\big)/\eta_t \\
&\leq 2G_X\lVert x_t-\widetilde{x}_{t+1}\rVert-B_{\phi}\big(\widetilde{x}_{t+1}, x_t^\phi\big)/\eta_t 
\leq \min\big\{2 D_X G_X, 2\eta_t G_X^2\big\},
\end{aligned}
\end{equation}
which implies that 
\begin{equation*}
\begin{aligned}
\left(\sum_{t=1}^{T}\delta_t^1\right)^2
&=\sum_{t=1}^{T}\left(\delta_t^1\right)^2+2\sum_{t=1}^{T}\delta_t^1\sum_{\tau=1}^{t-1}\delta_\tau^1 
=\sum_{t=1}^{T}\left(\delta_t^1\right)^2+2\sum_{t=1}^{T}\delta_t^1\bigg(\frac{L_\phi(D_X+C^1)}{\eta_t}-\epsilon\bigg) \\
&\leq \sum_{t=1}^{T}4G_X^2 D_X^2+\sum_{t=1}^{T}4G_X^2L_\phi(D_X+C^1).
\end{aligned}
\end{equation*}
This results in the following regret bound:
\begin{equation}
\label{eqpf:OptIOMDA-x-bound-part2}
\begin{aligned}
\regret^1_T\leq\epsilon+4G_X\sqrt{\left(D_X^2+L_\phi D_X+L_\phi C^1\right)T}.
\end{aligned}
\end{equation}
Combining \cref{eqpf:OptIOMDA-x-bound-part1,eqpf:OptIOMDA-x-bound-part2} yields
\begin{equation*}
\begin{aligned}
\regret^1_T\leq\epsilon+4\min\left\{\sum_{t=1}^{T}\rho(f_t,h_t), G_X\sqrt{\left(D_X^2+L_\phi D_X+L_\phi C^1\right)T}\right\}.
\end{aligned}
\end{equation*}
Likewise, the individual regret of Player~2 satisfies
\begin{equation*}
\regret^2_T\leq\epsilon+4\min\left\{\sum_{t=1}^{T}\rho(f_t,h_t), G_Y\sqrt{\left(D_Y^2+L_\psi D_Y+L_\psi C^2\right)T}\right\}. 
\end{equation*}
\end{proof}

\section{Proof of \cref{thm:OptOPPM}}

Refer to the proof of \cref{thm:OPPM}. 

\section{Proof of \cref{thm:MP-OptOPPM}}

\begin{proof}
Note that in the proof of \cref{lem:OptOPPM}, the relaxed form inequalities $\delta_t^1, \delta_t^2\leq 2\rho(f_t,h_t)$ are employed. 
In fact, $L_t$ induces a tighter upper bound. 
\begin{equation*}
\begin{aligned}
\delta_t^1&\leq f_t(x_t,y_t)-h_t(x_t,y_t)+h_t(\widetilde{x}_{t+1},y_t)-f_t(\widetilde{x}_{t+1}, y_t) \\
&\leq \sum_{k=1}^\kappa \omega_t^k \Big(\big|f_t(x_t,y_t)-h_t^k(x_t,y_t)\big| + \big|f_t(\widetilde{x}_{t+1}, y_t)-h_t^k(\widetilde{x}_{t+1},y_t)\big|\Big)
\leq 2\left\langle L_t,\omega_t\right\rangle.
\end{aligned}
\end{equation*}
Likewise, $\delta_t^2\leq 2\left\langle L_t,\omega_t\right\rangle$. 
Now the benign bound $V'_T$ in \cref{thm:OptOPPM} can be substituted with $\sum_{t=1}^T \left\langle L_t,\omega_t\right\rangle$. 
Drawing from \cref{lem:hedge}, we have that 
\begin{equation*}
\begin{aligned}
\sum_{t=1}^T\left\langle L_t, \omega_t\right\rangle
&\leq \sum_{t=1}^T\left\langle L_t, 1_{k}\right\rangle+2\sqrt{(1+\ln T)L_{\infty}\sum_{t=1}^T\left\langle L_t, 1_{k}\right\rangle}+O\left(\ln T\right) \\
&=\sum_{t=1}^T\left\langle L_t, 1_{k}\right\rangle+O\left(\sqrt{\ln T}\right)\!\sqrt{\sum_{t=1}^T\left\langle L_t, 1_{k}\right\rangle}+O\left(\ln T\right) \\
&\leq 2\sum_{t=1}^T\left\langle L_t, 1_{k}\right\rangle+O\left(\ln T\right) 
\leq2\sum_{t=1}^{T}\rho\left(f_t, h_t^k\right)+O\left(\ln T\right),\quad \forall k\in 1\!:\!d,
\end{aligned}
\end{equation*}
where $1_{k}$ denotes the $d$-dimensional one-hot vector with the $k$-th element being 1.
Given the arbitrariness of $k$, we obtain
\begin{equation*}
\begin{aligned}
\sum_{t=1}^T\left\langle L_t, \omega_t\right\rangle
\leq 2\min_{k\in 1:d}\sum_{t=1}^{T}\rho\left(f_t, h_t^k\right)+O\left(\ln T\right),
\end{aligned}
\end{equation*}
In conclusion, the benign bound $V'_T$ in \cref{thm:OptOPPM} can be substituted with 
\[\min_{k\in 1:d}\sum_{t=1}^{T}\rho\left(f_t, h_t^k\right)+O\left(\ln T\right).\] 
Ignoring the dependence on poly-logarithmic factors yields the conclusion to be proved. 
\end{proof}

\begin{lemma}[Static Regret for Clipped Hedge, Static Version of Corollary~B.0.1 of \citet{Campolongo2021closer}]
\label{lem:hedge}
Assume that $L_t\geq 0$, $\max_{t\in 1:T}\lVert L_t\rVert_\infty=L_\infty$ and $\alpha=d/T<1$. 
If the learning rate follow from $\theta_t=(\ln T)/\big(\epsilon+\sum_{\tau=1}^{t-1}\sigma_{\tau}\big)$, 
where the constant $\epsilon > 0$ prevent $\theta_1$ from being infinite, and $\sigma_t=\langle L_t, \omega_t-\omega_{t+1}\rangle-\mathrm{KL}(\omega_{t+1}, \omega_{t})/\theta_{t}$, 
Then the clipped Hedge enjoys the following static regret: 
\begin{equation*}
\begin{aligned}
\regret_T u\leq 2\sqrt{(1+\ln T)L_{\infty}\sum_{t=1}^T\left\langle L_t, u\right\rangle}+O\left(\ln T\right), \quad \forall u\in\bigtriangleup_d^0.
\end{aligned}
\end{equation*}
\end{lemma}
%
%

\section{Solver for Clipped Hedge}
\label{app:algo}
The clipped Hedge equivalent to the following update:
\begin{equation*}
\begin{aligned}
\omega_{t+1}=\arg\min_{\omega\in\bigtriangleup_d^\alpha}\left\langle\ln \frac{\omega}{\omega_{t}\cdot\exp(-\theta_{t}L_t)}, \omega\right\rangle,
\end{aligned}
\end{equation*}
Thus, an efficient solution is attainable by minor adjustments to the algorithm depicted in Figure~3 of \citet{herbster2001tracking}. The modified algorithm is elaborated in \cref{alg:negtive-entropy-optimization-solver}, wherein the primary alteration is the removal of the constraint $\lVert W\rVert_1=1$. 

\begin{algorithm}[t]
\captionsetup{font=small}
\caption{Program to Solve $\displaystyle w^*=\arg\min\nolimits_{w\in \bigtriangleup_d^{\alpha}} \left\langle\ln (w/W), w\right\rangle$}
\label{alg:negtive-entropy-optimization-solver}
\algrenewcommand\algorithmicrequire{\textbf{Input:}}
\algrenewcommand\algorithmicensure{\textbf{Output:}}
\begin{algorithmic}[1]
\Require $W$, $\alpha$
\State $d\gets |W|$,~~ $I\gets 1\!:\!d$,~~ $C_{\#}\gets 0$,~~ $C_{\%}\gets 0$
\While{$I\neq\varnothing$ }
\State $w\gets$ the median of $W_I$
\State $L\gets \{i\mid i\in I, W_i<w\}$,~~ $M\gets \{i\mid i\in I, W_i=w\}$,~~ $H\gets \{i\mid i\in I, W_i>w\}$
\If{$\displaystyle w\frac{d-(C_{\#}+|L|)\alpha}{\lVert W\rVert_1-\big(C_{\%}+\lVert W_L\rVert_1\big)}<\alpha$ }
\State $C_{\#}\gets C_{\#}+|L|+|M|$,~~ $C_{\%}\gets C_{\%}+\lVert W_L\rVert_1+\lVert W_M\rVert_1$
\If{$H=\varnothing$ }
\State $w\gets\min\{W_i\mid i\in I, W_i>w\}$
\EndIf
\State $I\gets H$
\Else 
\State $I\gets L$
\EndIf
\EndWhile
\Ensure $\forall i\in 1\!:\!d$,~~ $\displaystyle w_i^*\gets\begin{cases}\begin{aligned}&\frac{\alpha}{d}&W_i<w\\&\frac{W_i}{d}\frac{d-C_{\#}\alpha}{\lVert W\rVert_1-C_{\%}}&W_i\geq w\end{aligned}\end{cases}$
\end{algorithmic}
\end{algorithm}

\end{document}